\documentclass[10pt]{article}
\usepackage{float}

\PassOptionsToPackage{round,authoryear}{natbib}
\usepackage[final]{saints}          % line-numbered internal draft

\usepackage{wrapfig}
\usepackage{algorithm}
\usepackage{algorithmic}
\setcitestyle{authoryear,round,citesep={;},aysep={,},yysep={;}}
\usepackage{multirow}
\usepackage[table]{xcolor}
\usepackage{enumitem}
\usepackage{adjustbox}
\usepackage{aliascnt}

\usepackage{tikz}
\usetikzlibrary{patterns,calc,arrows.meta,positioning,intersections}
\usepackage[normalem]{ulem}
\usepackage{multirow}

\newcommand{\mymacro}[1]{#1}

\newcommand{\R}{\mymacro{\mathbb{R}}}

\newcommand{\N}{\mymacro{\mathbb{N}}}

\newcommand{\floor}[1]{\left\lfloor #1 \right\rfloor}

\newtheoremstyle{saintsplain}
  {6pt}
  {6pt}
  {\itshape}
  {}
  {\bfseries}
  {.}
  {0.5em}
  {\thmname{#1}\thmnumber{ #2}\thmnote{ \bfseries(#3)}}

\newtheoremstyle{saintsdefinition}
  {6pt}
  {6pt}
  {\normalfont}
  {}
  {\bfseries}
  {.}
  {0.5em}
  {\thmname{#1}\thmnumber{ #2}\thmnote{ \bfseries(#3)}}

\newcommand{\bits}{\{0,1\}}
\newcommand{\Mcal}{\mathcal{M}}

\newcommand{\W}{\mathbf{w}}
\newcommand{\bb}{\mathbf{b}}
\newcommand{\zz}{\mathbf{z}}
\newcommand{\What}{\widehat{\W}}

\newcommand{\ceil}[1]{\left\lceil #1 \right\rceil}

\newcommand{\Stir}{\mathrm{Stir}}

\newcommand{\enc}{\mathrm{enc}}

\renewcommand{\PaperVenueLine}{}
\renewcommand{\PaperTitle}{Algorithmic Simplification of Neural Networks with Mosaic-of-Motifs}
\renewcommand{\PaperLeadAuthors}{Pedram Bakhtiarifard}
\renewcommand{\PaperAuthors}{Pedram Bakhtiarifard, Tong Chen, Jonathan Wensh{\o}j, \\ Erik B. Dam, Raghavendra Selvan}

\renewcommand{\PaperAffiliations}{Department of Computer Science, University of Copenhagen, Denmark}
\renewcommand{\PaperDate}{April 2026}
\renewcommand{\PaperCorrespondence}{\href{mailto:\{pba,raghav\}@di.ku.dk}{\texttt{\{pba,raghav\}@di.ku.dk}}}
\renewcommand{\PaperCode}{\href{https://github.com/saintslab/MoMos}{\texttt{github.com/saintslab/MoMos}}}
\renewcommand{\PaperBibliographyStyle}{abbrvnat}

\begin{document}

\makePaperFrontmatter

\begin{abstract}
Large-scale deep learning models are well-suited for compression. Across a variety of tasks, methods like pruning, quantization, and knowledge distillation have been used to achieve massive reductions in model parameters with only marginal performance drops. This raises the central question: \emph{Why are deep neural networks suited for compression?} In this work, we take up the perspective of algorithmic complexity to explain this behavior. We hypothesize that the parameters of trained models have more structure and, hence, exhibit lower algorithmic complexity compared to the weights at (random) initialization. Furthermore, model compression methods harness this reduced algorithmic complexity to compress models. Although an unconstrained parameterization of model weights, $\mathbf{w} \in \mathbb{R}^n$, can represent arbitrary weight assignments, the solutions found during training exhibit repeatability and structure, making them simpler to implement than a trivial program. To this end, we formalize the Kolmogorov complexity of $\mathbf{w}$ by $\mathcal{K}(\mathbf{w})$. We introduce a constrained parameterization $\widehat{\mathbf{w}}$ that partitions parameters into blocks of size $s$ and restricts each block to be selected from a set of $k$ reusable motifs, specified by a reuse pattern (or mosaic). The resulting method, \textit{Mosaic-of-Motifs} (MoMos), provides a theoretically justified parameterization that biases optimization toward algorithmically simpler solutions. Empirical evidence from multiple experiments shows that MoMos consistently lowers the algorithmic complexity of neural networks during training while preserving the performance of unconstrained models. These results suggest that parameter compressibility is not only observed after training, but can be induced from the optimization domain.
\end{abstract}

\section{Introduction}
Better model performance is regularly achieved through scale. Large models, trained on internet-scale datasets with a substantial amount of compute, reliably deliver improved performance~\citep{kaplan2020scalinglawsneurallanguage, rosenfeld2019constructivepredictiongeneralizationerror}. We observe this trend across many domains, such as computer-vision~\citep{dosovitskiy2021imageworth16x16words}, multi-modal learning~\citep{radford2021learningtransferablevisualmodels}, and reinforcement-learning~\citep{vinyals2019grandmaster}. And although new methodologies and architectures undoubtedly contribute to these advances, performance increases consistently coincide with an increase in scale.

At the same time, scaling incurs substantial costs in memory, compute, and energy~\citep{patterson2021carbonemissionslargeneural,schwartz2019greenai}, which has motivated extensive work on compression techniques, including pruning~\citep{frankle2019lotterytickethypothesisfinding}, quantization~\citep{nagel2021whitepaperneuralnetwork}, distillation~\citep{hinton2015distillingknowledgeneuralnetwork}, and parameter sharing~\citep{pham2018efficientneuralarchitecturesearch}. This has demonstrated that modern models are often highly redundant and can be reduced by orders of magnitude with limited, or even no degradation in performance~\citep{han2016deepcompressioncompressingdeep}. Moreover, it reinforces that sheer scale is not necessarily the linchpin for advances in model capability.

A common notion used to reason about the scale of models is their complexity, which is often associated with parameter count, layer-wise parametrization, or computational cost~\citep{tan2020efficientnetrethinkingmodelscaling}. Such measures correlate with performance and can relate model error to scale~\citep{hestness2017deeplearningscalingpredictable}, yet they remain agnostic to how parameters are structured or reused across model parameters. In this framing, complexity reflects the size of the parameterization and its associated computational footprint, but not the structure of the learned parameters themselves --- although the objective of most compression techniques is precisely that of finding redundancies in structure and repeatability. This raises the central question of \emph{Why are models suited for compression?}

\begin{wrapfigure}{r}{0.42\textwidth}
\vspace{-1.1em}
\centering
\includegraphics[width=0.20\textwidth]{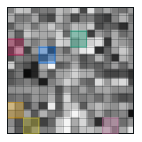}
\hfill
\includegraphics[width=0.20\textwidth]{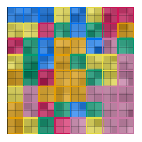}
\vspace{0.3em}
\caption{An unconstrained  \textbf{(left)} $16 \times 16$ parameterization with randomly selected motifs highlighted. MoMos parameterization \textbf{(right)} enforcing $2\times2$ motif reuse.}
\label{fig:momo}
\vspace{-1.0em}
\end{wrapfigure}

From algorithmic complexity theory, the complexity of an object is the length of its shortest generating program~\citep{solomonoff1960preliminary,kolmogorov1965three, chaitin1977algorithmic}. Programs refer to descriptions of procedures. We hypothesize that optimization generally increases algorithmic regularity, so trained models admit shorter descriptions than (at) random initialization. In our analysis, the object is not the dataset or the learned function, but the \emph{parameterization} itself --- that is, the weight matrix, and more specifically its structure rather than the specific numerical values. This differs from prior approaches, such as weight-sharing~\citep{ullrich2017soft} and quantization~\citep{hubara2016binarized, han2016deepcompressioncompressingdeep} that primarily constrain the space of parameter values. Interestingly, we theoretically show that the optimization domain can be shaped to induce redundancy with minimal performance impact, yielding more compressible parameters than in the unconstrained setting. This redundancy is illustrated in Fig.~\ref{fig:momo} (left) for a weight matrix at random initialization, which, from the perspective of algorithmic complexity, only has a near-trivial description. On the other hand, the weight matrix illustrated in Fig.~\ref{fig:momo} (right) admits a shorter description, as it uses repeating blocks of patterns resulting in algorithmically simpler objects.

Our goal is to impose such simplicity by constraining on repeatability and reuse in parameter space. By restricting the hypothesis class, we show that algorithmic simplicity can be induced, yielding provably shorter worst-case descriptions and admitting compressible models in practice.

From this perspective, our contributions are:

\begin{itemize} 
    \item We present \emph{Mosaic-of-Motifs (MoMos)}, which is a constrained parameterization of model weights; partitioning parameters into blocks of size $s$, it restricts each block to be selected from a set of $k$ reusable motifs, specified by a reuse pattern $\mathcal{M}$, referred to as a \emph{mosaic}.
    
    \item We show that MoMos parameterization guarantees lower worst-case algorithmic complexity, controlled by $s$ and $k$, and find that the corresponding bound is conservative  in practice under the observed learning dynamics, showing directions for future work.

    \item Using proxies of Kolmogorov Complexity, we empirically show that MoMos drives learned weights toward more algorithmically regular, more compressible, parameterizations than the unconstrained baselines with little to no drops in performance.
\end{itemize}

While model compressibility has been studied extensively, to the best of our knowledge, it is the first work to provide guarantees on the algorithmic complexity of learned parameters by constraining the model class. In particular, this work is agnostic to any compression method and instead studies compressibility as a property of the optimization domain induced by model parameterization.
\section{Related work}
It is demonstrated that trained models contain substantial redundancy in their parameter space and are therefore needlessly complex, with complexity characterized by parameter count, compute and memory cost~\citep{nagel2021whitepaperneuralnetwork,han2016deepcompressioncompressingdeep,hinton2015distillingknowledgeneuralnetwork}. Vector quantization~\citep{gray1984vector} approaches exploit this redundancy by constraining the latent space to discrete representations~\citep{van2017neural} or by quantizing weights into fixed vectors and reusing repeated matrix-multiplications to improve downstream efficiency~\citep{gong2014compressing,vali2023stochastic}.
Similarly, redundancies are identified in weight-sharing approaches to cluster weights, reducing the number of distinct parameters and compressing networks~\citep{wu2018deepkmeansretrainingparameter,softweightsharing,ullrich2017soft}. In contrast, our work is closer to weight-space understanding~\citep{han2026survey} than to classical compression methods; compressibility is here measured as an effect rather than optimized as an objective.

Entropy measures have been proposed as proxies for model complexity by measuring the distribution of weights and activation functions, relating it to generalization behavior~\citep{xu2017information,achille2018emergenceinvariancedisentanglementdeep}, and also to guide optimization trajectories~\citep{chaudhari2019entropy}. However, when weight arrangements induce similar distributions, thus similar entropy values, these proxies cannot, in general, capture algorithmic regularities. 

The Minimum Description (MDL) principle argues that the shortest description of data relative to a model class, and compressor is the best model~\citep{RISSANEN1978465,grunwald2004tutorialintroductionminimumdescription,grunwald2008algorithmic}. A distinction to our work is treating the parameterization of model weights as the object whose description length we control, independent of any particular dataset and compressor.

\begin{wrapfigure}{r}{0.42\textwidth}
\vspace{-1.2em}
\centering
\includegraphics[width=\linewidth]{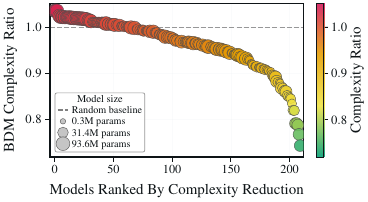}
\vspace{-0.6em}
\caption{Algorithmic complexity reduction for randomly sampled pretrained \textsc{timm} models with up to 100M parameters, measured as the ratio of pretrained to initialized complexity using Eq.~\eqref{eq:bdm}. The 200 models are ordered by increasing complexity reduction.}
\label{fig:bdm_complexity}
\vspace{-1.0em}
\end{wrapfigure}

\paragraph{Algorithmic Complexity of Finite Objects.}
Algorithmic information theory defines complexity as the description length of an object itself. The Kolmogorov complexity~\citep{solomonoff1960preliminary,kolmogorov1965three,
chaitin1977algorithmic} of a finite binary string $\W$, denoted $\mathcal{K}(\W)$, is the length of the shortest program $\mathbf{p}$ that outputs $\W$ and halts, when run on a universal prefix Turing machine. Let $U$ be such a universal prefix Turing machine. For a finite binary string $\W \in \bits^*$, and program $\mathbf{p}$, its Kolmogorov complexity is defined as
\begin{equation}
\mathcal{K}(\mathbf{w}) := \min\{|\mathbf{p}| : U(\mathbf{p})=\mathbf{w}\},
\end{equation}

where $|\mathbf{p}|$ is the bit-length of the program. This computation, however, requires deciding whether all shorter programs halt with output $w$, fail, or run forever, which is undecidable by the halting problem~\citep{davis1982computability}; hence, it is not computable.

While computing the exact Kolmogorov complexity is not tractable, it relates to the algorithmic probability of an object, $\mathcal{P}(\W) = \sum_{\mathbf{p}_{U(\mathbf{p})=\W}}2^{|-\mathbf{p}|}$, through the Coding Theorem~\citep{levin1974laws} by $\mathcal{K}(\W)= -\log\mathcal{P(\W)}+O(1)$; the probability of randomly generating a program that outputs $\W$ and halts. 
For small binary objects, the Coding Theorem Method (CTM)~\citep{zenil2015two} operationalizes this idea, approximating $\mathcal{P}(\W)$ by the output frequency of $\W$ on large-scale simulations of Turing machines~\citep{soler2014calculating}. Simulations are only practical for smaller programs, but are extended to larger objects by the Block Decomposition Method (BDM)~\citep{zenil2018decomposition} by summing the frequencies of the blocks of decomposed objects $\{\mathbf{b_i}\} \in \W$. The Kolmogorov Complexity measure using the BDM is expressed as 
\begin{equation}
    \mathcal{K}_B(\W) := \sum_{i} \mathrm{CTM}(\textbf{b}_i)+\log_2(m_i)
    \label{eq:bdm}
\end{equation}
where $\mathrm{CTM}(\mathbf{b}_i)$ is the CTM estimate of the distinct block $\mathbf{b}_i$, and $m_i$ is its multiplicity in $\W$. Central to this idea is the observation that simple objects occur often and are therefore of low complexity; we note that repeated blocks incur only a logarithmic repetition cost, whereas distinct blocks must be described separately.

These techniques have been used to study regularities in learned representations that are not reflected by statistical measures~\cite{zenil2020algorithmic}, with recent evidence suggesting that CTM-based measures more closely track training dynamics in binarized neural networks~\cite{sakabe2025evaluating}. Furthermore, analysis of gradient-based training suggests that optimization can be biased towards lower complexity and shorter descriptions~\cite {wilson2025deep,goldblum2024position}, which aligns with the idea that reuse reduces description length. Fig.~\ref{fig:bdm_complexity} provides further evidence that models are compressible because training induces repeatable structure in parameter space, yielding shorter descriptions and lower algorithmic complexity as captured by CTM-based complexity.

\section{Mosaic-of-Motifs (MoMos)}\label{sec:methods}
Our approach imposes repeatability by design and reduces the algorithmic complexity of neural network weights, thereby simplifying the model. We show this by giving a uniform, worst-case upper bound on the description length of any parameterization in its hypothesis class. This not only improves the trivial description in Eq.~\eqref{eq:kc_trivial}, but also shows explicitly how the simplification induced by MoMos controls the geometry of the optimization domain.

Let $f_\W$ denote a model parameterized by weights
$\W=\{\W_\ell\}_{\ell=1}^L$, where $\W_\ell \in \R^{n_\ell}$ are structured components such as convolutional kernels or linear layers, and $n=\sum_{\ell=1}^L n_\ell$ is the total number of parameters over $L$ layers. Without loss of generality, we flatten all parameters into a single vector $\W \in \mathbb{R}^n$. Let $q \in \mathbb{N}$ denote the number of bits per weight, then the bit-string representing the weights is given as $\W \in \bits^{nq}$. For any $\W$, there is a trivial description that lists the $nq$ bit-string, which gives the worst-case upper bound
\begin{equation}\label{eq:kc_trivial}
\mathcal{K}(\W) \le nq + O(1) .
\end{equation}
This holds for every $\W$; especially for \emph{unstructured}, i.e., any description requiring at least $|\W|$ bits.

\begin{definition}[MoMos Hypothesis Class]
Choose a block size $s \in \N$ and let $m = \ceil{n / s}$ be the number of blocks\footnote{In practice, we use a deterministic function that maps $\W$ to $\mathbb{R}^{s \cdot m}$ by padding $s\cdot m - n$ trailing zeros.}. Define an invertible slicing operator $\psi : \R^n \to \R^{s\cdot m}$ and $\psi(\W) = (\bb_1,\dots,\bb_m)$ where $\bb_i \in \R^s$ are each blocks of length $s$. The inverse $\psi^{-1}$ satisfies
$\psi^{-1}(\psi(\W)) = \W$, $\forall \W \in \mathbb{R}^n$.
For $1 \le k \le m$, a MoMos parameterization consists of fixed blocks $\mathcal{Z} = (\zz_1,\dots,\zz_k)$ with $\zz_i \in \R^s$ referred to as \emph{motifs}, and a reuse pattern $ \Mcal \in [k]^m$, referred to as a \emph{mosaic} where $[k] = \{1,\dots,k\}$. 

A mosaic decides the arrangement of motifs $\widehat{\bb}_i = \zz_{\Mcal(i)}$ and replaces $\bb_i$ with a motif from $\mathcal{Z}$ for $i \in  [m]$. The \emph{mosaic-of-motifs} then defines all the parameters of a model by
\begin{equation}
 \What =\psi^{-1}(\widehat{\bb}_1,\dots,\widehat{\bb}_m)  
 = \psi^{-1}(\zz_{\Mcal(1)},\dots,\zz_{\Mcal(m)}) \in \R^{n}.
\end{equation}
The hypothesis class induced by the $(s,k)$-MoMos parameterization is
\begin{equation}\label{eq:def_HSK}
    \mathcal{H}_{s,k} = \{\What: \mathbf{z}_i \in \mathbb{R}^s, \; \mathcal{M} \in [k]^m\} \subseteq \mathbb{R}^n ,
\end{equation}
denoting the set of all parameterizations that can be represented as $m$ blocks, each selected from a fixed set of $k$ motifs with their arrangement defined by $\Mcal$, the mosaic.
\end{definition}

In particular, $\mathcal{H}_{s,k}$ restricts the optimization domain by replacing arbitrary blocks with reused motifs, inducing structure. Here, $k$ controls how many motifs are available, while $s$ controls the size of the motifs. When $k=m$, the class can represent arbitrary block choices, whereas a smaller $k$ forces more blocks to share the same motif. One caveat here is that parameter values can still be arbitrary inside each motif, so $\mathcal{H}_{s,k}$ can induce a rich optimization domain while still imposing regularity. We next bound the worst-case description length of any parameterization in $\mathcal{H}_{s,k}$, improving Eq.~\eqref{eq:kc_trivial}. The relative bound of the improvement as $s$ increases with fixed $k$ can be seen in Fig.~\ref{fig:momos_bound}.
\begin{wrapfigure}{r}{0.40\textwidth}
\vspace{-0.0em}
\centering
\includegraphics[width=1\linewidth]{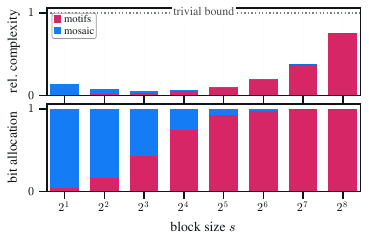}
\vspace{-1.2em}
\caption{The algorithmic complexity makeup attributed to motifs vs. mosaic arrangement. Bit allocation \textbf{(bottom)} share of motifs contributes most of the still lower \textbf{(top)} relative complexity.}
\label{fig:momos_bound}
\vspace{-2.0em}
\end{wrapfigure}

\begin{proposition}[MoMos Complexity Bound]
\label{prop:momo_kc}
There exists a constant $\zeta$, such that 
\begin{equation}
\mathcal{K}(\What) \le ksq + m\ceil{\log_2 k} + \zeta, \quad \forall \ \What \in \mathcal{H}_{s,k}.
\label{eq:momos-comp}
\end{equation}
\end{proposition}

This is straightforward to prove, as shown in Sec.~\ref{app:proof_upper_momo}, where we do this by construction of a program that reconstructs $\What$ from a set of motifs and a mosaic.

\begin{corollary}[Lower Complexity Regime]
\label{cor:nonvacuous}
Whenever $ksq + m\ceil{\log_2 k} < nq$, we have
\begin{equation}
\mathcal{K}(\widehat{\W}) < nq + \zeta, \quad \What \in \mathcal{H}_{s,k}
\end{equation}
In particular, the description in Prop.~\ref{prop:momo_kc} is strictly shorter than the trivial description of length $nq$ bits, up to the same additive constant, $\zeta$.
\end{corollary}

\paragraph{Granularity and Repeatability}
We are interested in the regime where $k \ll m$, since extensive motif reuse imposes regularity and yields parameterizations that are easier to describe algorithmically. Each motif of size $s$ controls the granularity of reuse. Smaller motifs allow a richer optimization domain, as illustrated in Fig.~\ref{fig:momos_domain}, whereas larger motifs enforce greater regularity and a less rich optimization domain. However, a similar function can be realized in very different areas of the optimization domain~\cite{dinh2017sharp}. Each motif is the unit of reuse and can otherwise take arbitrary values, yet in practice we enforce $\zz_1 = \mathbf{0}$, such that repeating $\zz_1$ can induce more sparsity; however, we observe a less than 1\% sparsity across all configurations and settings. This suggests that the reduction in algorithmic complexity and model compressibility is not naturally driven by learning dynamics that lead to sparsification discussed in Sec.~\ref{sec:sparsity}.

\subsection{Geometry of the Constrained Optimization Domain}
\begin{wrapfigure}{r}{0.40\textwidth}
\vspace{-2.1em}
\centering
\includegraphics[width=1\linewidth]{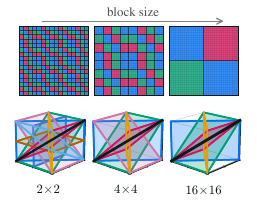}
\vspace{-1em}
\caption{The effect of block size granularity on domain geometry. Larger \textbf{(top)} blocks subsume the flexibility of local patterns impacting the richness of the \textbf{(bottom)} optimization subspaces of the domain; algorithmic simplification is exchanged for function multiplicity.}
\label{fig:momos_domain}
\end{wrapfigure}

We have seen how MoMos enforces regularity using repeated motifs. Its constraints do not directly change the parameter space $\mathbb{R}^n$; however, they restrict the subset over which optimization is performed. These restrictions affect the \emph{richness} of the resulting optimization domain, illustrated in Fig.~\ref{fig:momos_domain}, which we examine through a decomposition of its induced geometry.

Fix a mosaic $\Mcal \in [k]^m$, define the slice of $H_{s,k}(\Mcal) \subseteq \mathbb{R}^n $ as the set of weights whose block decomposition $\psi(\What) = (\widehat{\bb}_1, \ldots, \widehat{\bb}_m)$ satisfies
\begin{equation}
    \widehat{\bb}_i = \widehat{\bb}_j \quad \text{whenever} \quad \Mcal(i) = \Mcal(j).
\end{equation}
That is, blocks with the same label under $\Mcal$ are equal.

\begin{lemma}[Optimization Subspaces]\label{lemma:linear_component}
Fix $\Mcal \in [k]^m$ and let $k_\Mcal = |\{\Mcal(1),\ldots,\Mcal(m) \}|$ be the number of distinct block variables used by $\Mcal$. Then $\mathcal{H}_{s,k}(\Mcal) \subseteq \mathbb{R}^n$ is a linear subspace. If $s \mid n$, then
\begin{equation}
    \dim(\mathcal{H}_{s,k}(\Mcal)) = sk_\Mcal \le sk .
\end{equation}
    In general, $\dim(\mathcal{H}_{s,k}(\Mcal)) \le sk_\Mcal \le sk$ when $s \nmid n$. Proof in Sec.~\ref{app:linear_component}.
\end{lemma}

\begin{proposition}[Domain Decomposition.]
\label{prop:union_subspaces}
The MoMos hypothesis class decomposes as
\begin{equation}
\mathcal{H}_{s,k} = \bigcup_{\Mcal \in [k]^m} \mathcal{H}_{s,k}(\Mcal).
\end{equation}
Since $|[k]^m| = k^m$, $\mathcal{H}_{s,k}$ is a union of at most $k^m$ linear components, each of dimension at most $sk$.
\end{proposition}
\begin{proof}
Every $\What \in \mathcal{H}_{s,k}$ corresponds to some mosaic $\Mcal\in[k]^m$ and motif values; fixing $\Mcal$ and varying the motifs gives $\mathcal{H}_{s,k}(\Mcal)$. Taking the union over all $\Mcal$ yields the decomposition, since $|[k]^m|=k^m$ and each $\mathcal{H}_{s,k}(\Mcal)$ has dimension at most $sk$ by Lem.~\ref{lemma:linear_component}.
\end{proof}

\label{sec:distinct}
Decomposing the optimization domain, we note that distinct mosaics do not necessarily define distinct subspaces; the symbols $[k]$ are only labels for shared (repeated) motifs. Although relabeling the symbols changes $\Mcal$, it can leave the induced structure unchanged. For instance, take $m=4, k=2, \Mcal=(1,2,1,2)$ and $\Mcal' = (2,1,2,1)$. Here the mosaics are different, yet both impose $\widehat{\bb}_1 = \widehat{\bb}_3$ and $\widehat{\bb}_2 = \widehat{\bb}_4$, hence $\mathcal{H}_{s,2}(\Mcal) = \mathcal{H}_{s,2}(\Mcal')$. To avoid overcounting mosaics that differ only by relabeling $[k]$, we consider each mosaic $\Mcal$ by the shared block positions it induces.

Assume $s \mid n$ so $m = n/s$. A mosaic $\Mcal \in [k]^m$ assigns to each block position $i \in [m]$ a motif index $\Mcal(i) \in [k]$. For each $j \in [k]$, define
\begin{equation}
    P_j(\Mcal) := \{i \in [m] : \Mcal(i) = j \}.
\end{equation}
This is the set of block positions that use index $j$. Some indices may be unused, so we only keep non-empty classes:
\begin{equation}
    \mathcal{P}(\Mcal) = \{P_j(\Mcal) : P_j(\Mcal) \neq \emptyset \} .
\end{equation}
Then $\mathcal{P}(\Mcal)$ is a partition of $[m]$ into $k_\Mcal$ parts, specifying which blocks are constrained to be equal.

\begin{proposition}[Distinctness Criterion of Mosaics]\label{prop:distinctness}
    Assume $s \mid n$. For any $\Mcal, \Mcal' \in  [k]^m$,
    \begin{equation}
        \mathcal{H}_{s,k}(\Mcal) = \mathcal{H}_{s,k}(\Mcal') \iff \mathcal{P}(\Mcal) = \mathcal{P}(\Mcal').
    \end{equation}
    For intuition, note the slice $\mathcal{H}_{s,k}(\Mcal)$ is determined by which equalities $\widehat{\bb}_i = \widehat{\bb}_j$ are enforced; $\mathcal{P}(\Mcal)$ encodes the same equalities without labels. Proof in Appx.~\ref{app:distinctness_criterion}.
\end{proposition}

This notion of distinctness determines the \emph{richness} of the optimization domain. Following the previous example, $\mathcal{P}(\Mcal) = \mathcal{P}(\Mcal') = \{\{1,3 \}, \{2,4 \} \}$, so although the number of mosaics is bounded by $k^m = 2^4 = 16$, there are less distinct subspaces in the optimization domain.

\begin{definition}[Distinct Subspaces in Mosaics]
    Assume $s \mid n$. Let $m = n/s$. For a mosaic $\Mcal \in [k]^m$, there are $k^m$ components, \emph{some} of which induce the same subspace. Let $\alpha(m,k)$ denote the number of distinct subspaces, 
    \begin{equation}
        \alpha(m,k) := |\{\mathcal{H}_{s,k}(\Mcal) : \Mcal \in [k]^m \}|.
    \end{equation}
    Let $\mathrm{Stir}(m,i)$ denote the Stirling number of the second kind, i.e., the number of ways to partition a set of $m$ objects into $i$ non-empty subsets~\citep{moser1958stirling}. Then by Prop.~\ref{prop:distinctness}, distinct subspaces correspond to partitions of $\{1,\dots,m\}$ into at most $k$ subsets, hence for $k \ge 2$,
    \begin{equation}
        \alpha(m,k) = \sum_{i=1}^k \mathrm{Stir}(m,i) < k^m .
    \end{equation}
    Moreover, $\alpha(m,k)$ is non-decreasing in both $m$ and $k$, so increasing the number of blocks $m$, or allowing more motifs $k$, cannot reduce the number of distinct components.
\end{definition}

\paragraph{Richness of The Optimization Domain.}\label{opt:rich} 
For fixed $(s,k)$, different mosaics $\Mcal \in [k]^m$ impose different sets of equalities for blocks $\widehat{\bb}_i = \widehat{\bb}_j$ across the $m$ block positions, thereby inducing distinct linear components $\mathcal{H}_{s,k}(\Mcal)$. Specifically, $\alpha(m,k)$ counts the number of such components that are distinct and how it enriches the optimization domain with more functional degrees of freedom.

To see the effect of $k$ on the optimization domain, note that increasing $k$ cannot decrease $\alpha(m,k)$, because it permits additional partitions of $[m]$. Decreasing $s$ increases $m = \ceil{n/s}$, which increases the number of possible linear components in the optimization landscape. For $m \ge 2$ and $k \ge 2$ we have $\alpha(m,k) \ge \Stir(m,1) + \Stir(m,2) = 2^{m-1}$, so the number of distinct components grows exponentially in $m$. To isolate this growth independent of the motif budget, assume $s \mid n$ and let $s\cdot k = c'$ be a constant factor. Then $\alpha(m,k) = \alpha({n}/{s}, {c'}/{s}) =: \beta(s)$ is a function of the block size.

\begin{figure}[htb]
    \centering
    \includegraphics[width=1\linewidth]{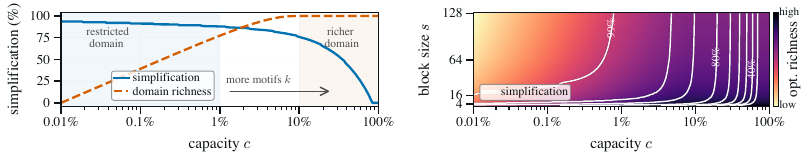}
    \caption{Trade-off between algorithmic simplification and richness of the optimization domain for Tiny-ViT~\citep{tinyvit} using MoMos. For a fixed block size $s=4$ \textbf{(left)}, capacity $c=k/m$ is the fraction of blocks retained as motifs; increasing $c$ enriches the optimization domain while reducing the amount of algorithmic simplification. Across block sizes $s$ \textbf{(right)}, darker colors indicate a richer optimization domain, and white contours mark fixed areas of simplification.}
    \label{fig:simplication}
\end{figure}

\begin{proposition}[Motif Granularity Increase Richness]\label{prop:fixed_budget}
    Fix $n, c'$ and let $s_1, s_2 \in \mathbb{N}$ with $s_1 \mid n$ and $s_2 \mid n$ and $s_1 \le s_2$. Then,
    \begin{equation}
        \beta(s_1) \ge \beta(s_2) .
    \end{equation}
    That is, $\beta(s)$ is monotonically decreasing in $s$.
\end{proposition}

\begin{proof}\label{proof:fixed_budget}
    Let $m_i = n/s_i$ and $k_i = c'/s_i$ with $s_i \mid c'$ for $i \in \{1,2 \}$. Since $s_1 \le s_2$, we have $m_1 \ge m_2$ and $k_1 \ge k_2$. Because $\alpha(m,k)$ is increasing in both $m$ and $k$
\begin{equation}
    \alpha(m_1,k_1) \ge \alpha(m_2, k_1) \ge \alpha(m_2, k_2) .
\end{equation}
Substituting $\beta(s_i) = \alpha(m_i,k_i)$ gives $\beta(s_1) \ge \beta(s_2)$.
\end{proof}

\section{Evaluating Characteristics of MoMos Parameterization}
\label{sec:experiments}
We empirically characterize how parameters of the MoMos hypothesis class affect performance and compressibility, focusing on whether accuracy varies predictably with algorithmic complexity as constraints tighten. Our experiments span ResNet20~\citep{resnet}, Tiny-ViT~\citep{tinyvit}, and a Multi-Layered Perceptron (MLP) with 5 hidden layers of 256 neurons per layer, on CIFAR-10~\citep{cifar10}. We also scale to more recent models on different tasks by examining MoMos on models Swin-B~\citep{liu2021swin} on Tiny ImageNet~\citep{le2015tiny} and SmolLM2-360M~\citep{allal2025smollm2} on WikiText-103~\citep{merity2016pointer}. Finally, we also compare against five baseline methods (some) that are similar in spirit, but not necessarily model-agnostic, similar to MoMos. Full experimental setup is detailed in Appx.~\ref{app:exp_setup}.

In addition to standard performance metrics, we report and denote by $r_s$ the relative algorithmic simplification (RAS) as the ratio between the upper bound description length for unconstrained optimization (Eq.~\ref{eq:kc_trivial}) and the MoMos parameterization (Eq.~\ref{eq:momos-comp}), ignoring constants. We also measure the algorithmic complexity of trained neural network weights ($\W^*$) compared to randomly initialized weights ($\W_0$), denoted as $r_\mathrm{BDM}$, as measured in Eq.~\eqref{eq:bdm}. We report the RAS of baseline methods from each method's own deployed representation. These are defined as follows:
\begin{equation}
r_\mathrm{BDM} := 
\frac{\mathcal{K}_B(\W^*)}{\mathcal{K}_B(\W_0)} , ~~~~  
r_s : \approx 
\frac{nq}{ksq + m\ceil{\log_2 k}} .
\label{eq:rac}
\end{equation}
Due to the practical limits of finite-Turing machine simulations we discuss in Sec.\ref{sec:limitations}, $r_{\mathrm{BDM}}$ becomes uninformative for high-dimensional models, collapsing toward entropy, as noted in~\cite{sakabe2025evaluating}. We therefore use RAS and relative motif capacity $c=k/m$ to compare across model scales. This changes the requested capacity by at most $1/m$, since $k=\floor{c\cdot m}$ implies $c-(k/m)<1/m$. The main experimental results can be seen in Tab.~\ref{tab:results_main} and additional details in Appx.~\ref{app:baseline_details}.

\begin{wrapfigure}{r}{0.36\textwidth}
\vspace{-1.3em}
\centering
\includegraphics[width=\linewidth]{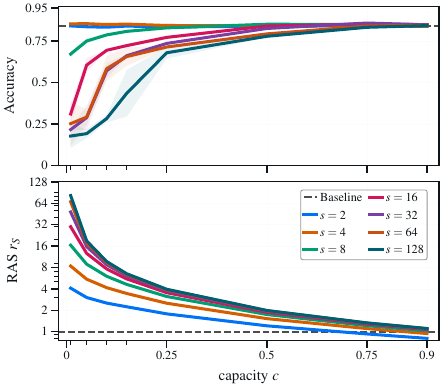}
\vspace{-1.4em}
\caption{Ablation over capacity and block size for Tiny-ViT on \textbf{(top)} test accuracy and \textbf{(bottom)} RAS.}
\label{fig:ablation_vit}
\vspace{-1.2em}
\end{wrapfigure}

\paragraph{Redundancy as a Learning Constraint.}
Fig.~\ref{fig:ablation_vit} isolates the effects of block size $s$ and capacity $c$. Increasing $c$ improves accuracy because the model can use more motifs, but it reduces RAS because the description length increases. Increasing $s$ gives stronger simplification, since each motif covers more weights, but also makes the optimization space more restrictive. Consistent with Eq.~\eqref{opt:rich}, small blocks, especially $s\in\{2,4\}$, preserve accuracy close to the full-precision baseline even at low capacity. Larger blocks lose accuracy more quickly at low capacity and require larger $c$ to recover. The same trend appears for the other architectures in Appx.~\ref{fig:full_ablation}.
Tab.~\ref{tab:ft_results} shows that this trade-off also holds when fine-tuning larger pretrained models. For Swin-B and SmolLM2, MoMos recovers baseline-level accuracy or comparable perplexity in the less restrictive settings while still achieving RAS gains. Tab.~\ref{tab:results_overhead} compares MoMos with standard compression baselines. MoMos reaches baseline accuracy on Tiny-ViT, with only about $2\%$ training-time overhead.

\section{Insights from Algorithmic Simplification}\label{sec:insights}
\begin{wraptable}{r}{0.38\textwidth}
\vspace{-3.5em}
\centering
\small
\setlength{\tabcolsep}{4pt}

\begin{tabular}{lccc}
\toprule
Epoch & Top $1\%$ & Top $10\%$ & Top $25\%$ \\
\midrule
1   & $3.09$ & $21.79$ & $44.30$ \\
200 & $5.39$ & $34.59$ & $64.19$ \\
\bottomrule
\end{tabular}

\vspace{0.3em}
\caption{Mosaic share of most used motif for Tiny-ViT $(s=4,c=0.1)$.}
\label{tab:motif_usage}
\vspace{-1.0em}
\end{wraptable}

\paragraph{Learning Concentrates Structure.}
MoMos achieves baseline accuracy at moderate capacity, for example, Tiny-ViT with $s=2,c=5\%$ attains $86.99\%$ accuracy compared to $86.56\%$ for the baseline with RAS $3.02$. Similar trends are observed for ResNet20 and the MLP, consistent with the ablation results in Appx. Fig.~\ref{fig:full_ablation}. For the MLP, where $r_\mathrm{BDM}$ remains informative, all configurations exhibit reduced $r_\mathrm{BDM}$, with the lowest-capacity setting approaching a $2\times$ reduction, consistent with the view in Fig.~\ref{fig:bdm_complexity}. Interestingly, this effect is observed in motif usage for MoMos; motif assignments become highly concentrated during training, with a small fraction of motifs accounting for most assignments (Tab.~\ref{tab:motif_usage}), while many motifs remain unused. This suggests that our bound in Prop.~\ref{prop:momo_kc} is conservative in practice and provides insights into future directions. Additional insight into this concentration is seen in Fig.~\ref{fig:motif_concentration_s4} and Appx.~\ref{app:additional_results} (Fig.~\ref{fig:motif_concentration_s2}).

\begin{wrapfigure}[8]{r}{0.38\textwidth}
\vspace{0.1em}
\centering
\includegraphics[width=0.9\linewidth]{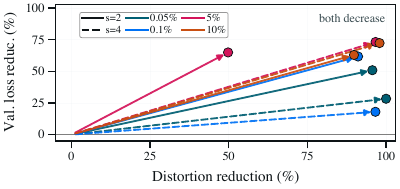}
\vspace{-0.7em}
\caption{Tiny-ViT distortion and validation loss at convergence.}
\label{fig:distortion_main}
\vspace{-0.0em}
\end{wrapfigure}

\paragraph{Structure Remains Compatible with Learning.}
MoMos introduces an approximation between the current dense weights and their motif-based reconstruction, as discussed in Sec.~\ref{sec:distinct}. We measure this along the training trajectory using the same trained model. At step $t$, the distortion is $D_t=\sum_b \|B_{t,b}-\hat{B}_{t,b}\|_F^2$, where $B_{t,b}$ is the current dense block and $\hat{B}_{t,b}$ is its MoMos reconstruction at the same step. If the imposed regularity were to negatively impact the model's ability to learn, this distortion would remain high or decrease only at the expense of validation performance. Instead, Fig.~\ref{fig:distortion_main} shows that distortion decreases while validation loss also decreases. As training proceeds, the weights move closer to a regular representation without harming optimization. Similar behavior is observed across other models given in Appx.~\ref{fig:distortion_all}.
\looseness=-1

\paragraph{Similarity of Decision Geometry.}
The loss landscape in Fig.~\ref{fig:decision_geometry} shows that algorithmic simplification does not simply preserve the dense solution in parameter space. MoMos solutions reach different regions of the local domain geometry, yet the slices show similar decision geometry. This suggests that MoMos is not compressing a fixed, dense solution but is steering optimization toward a different parameterization that realizes a comparable function while enforcing reusable structure. In line with work on loss landscapes and functionally similar solutions~\citep{li2018visualizing,draxler2018essentially}, this suggests that MoMos exploits the multiplicity of parameter realizations, biasing learning toward those with lower motif-based description length. This also aligns with the view that neural networks are biased towards functions of lower Kolmogorov complexity~\citep{valledeep}.

\begin{figure}[tbp]
    \centering
    \includegraphics[width=0.9\linewidth]{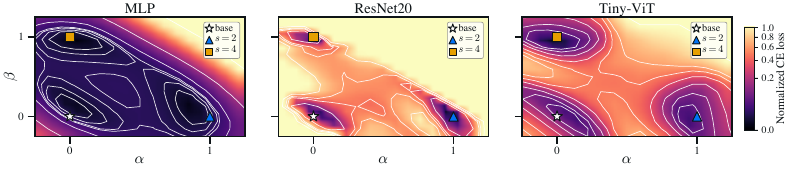}
    \vspace{-0.35cm}
    \caption{Loss landscape across architectures. For each model, we take the converged baseline parameters and two MoMos solutions with $s=2$ and $s=4$ at fixed capacity $c=0.1$, and evaluate validation cross-entropy on the plane spanned by these three parameter vectors. The coordinates $(\alpha,\beta)=(0,0)$, $(1,0)$, and $(0,1)$ correspond to the baseline, $s=2$, and $s=4$ solutions, respectively. Color shows CE loss normalized within each panel, and contours show fixed areas of loss.}
    \label{fig:decision_geometry}
    \vspace{-0.3cm}
\end{figure}

\section{Discussion}\label{sec:discussion}

\textbf{Similarities to Vector Quantization.} The MoMos parameterization in Algo.~\ref{algo:momos} is fundamentally similar to vector quantization~\citep{gray1984vector} (VQ) and is conceptually akin to weight sharing~\citep{ullrich2017soft} (WS). These approaches replace parts of the weight vector with a smaller set of shared values by either assigning blocks to codes from a learned codebook to minimize distortion~\citep{gong2014compressing}, or by directly tying network parameters. In MoMos, the codebook analog is the motif set; unlike prior work, we select this set globally across weights. And more importantly, we sample motifs at random and provide a better worst-case bound on algorithmic complexity (Prop.~\ref{fig:momos_bound}). This changes the question from tuning shared weights to identifying the reuse patterns (structure) that enable algorithmically simpler neural networks. Across several experiments, we show that such models are comparable to the unconstrained ones.

\textbf{Parameter Reuse Without Weight Generation.}
Recent weight-sharing methods learn shared filters~\citep{deeplyshared} or parameter generators~\citep{wang2023compact} to reduce the effective parameter count of neural networks. These methods assume that neural networks contain redundancy, identify useful sharing mechanisms, and reduce complexity in terms of parameter count. MoMos, in contrast, deliberately imposes this redundancy during learning and studies what happens to the model class under reuse constraints with arbitrary weights. Our analysis in Sec.~\ref{sec:methods} shows how the motif constraint changes and controls the optimization domain, while giving a tighter worst-case description-length bound for all parameterizations in the class. We then show in Sec.~\ref{sec:insights} that MoMos gives a concrete hypothesis class for studying algorithmic simplicity-biased learning~\citep{shwartz2017opening,valledeep}, as characterized in Prop.~\ref{prop:union_subspaces}.
\begin{wrapfigure}{r}{0.5\textwidth}
    \vspace{-0.25cm}
    \centering
    \includegraphics[width=1\linewidth]{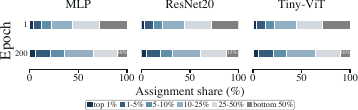}
    \vspace{-0.45cm}
    \caption{Motif assignment concentration is shown for MoMos with $s=2, c=0.1$. From epoch 1 to 200, assignment mass shifts toward the most-used motifs; the bottom 50\% receive only a small fraction of assignments at convergence. 
    }
    \label{fig:motif_concentration_s2}
    \vspace{-0.3cm}
\end{wrapfigure}
\textbf{Dense Models Can Still Become Algorithmically Simpler.}
\label{sec:sparsity}
The fraction of motifs with zero weights remains below $1\%$ in all the settings, indicating that the learned structure is dense rather than sparse. Simultaneously, as shown in Fig.~\ref{fig:motif_concentration_s2}, during training, the mosaic biases toward a small subset of motifs. This suggests the learned model function does not impose a fixed transformation of the network weights. Throughout learning, we observe that additional structure emerges from the imposed constrained on the optimization domain. Since motifs are selected randomly, we avoid confounding the effect of optimization heuristics with that of learning a better codebook; this shows that dense models can become simpler without becoming sparse.
\looseness=-1

\textbf{The Role of Scale.}\label{discussion:parameter_reuse}
BDM identifies repeating blocks in a bit string and provides a practical approximation of Kolmogorov complexity ($\mathcal{K}_B$ in Eq.~\eqref{eq:bdm}). We hypothesized that enforcing this kind of structure during neural network parameterization would yield algorithmically simpler models. MoMos does this by projecting weight blocks onto random motifs, which restricts optimization to mosaic-induced subspaces that approximate the unconstrained search space. We see evidence of this in Fig.~\ref{fig:decision_geometry}, where MoMos solutions suggest an algorithmically simpler realization of the unconstrained model behavior. Prop.~\ref{prop:fixed_budget} shows that, for a fixed budget, smaller blocks create a richer optimization space by giving the model access to more mosaics. This is also observed in Fig.~\ref{fig:ablation_vit}, where smaller blocks keep comparable performance across a broad range of capacities, while larger blocks require more capacity. This suggests that performance is best improved at low algorithmic complexity by enriching the optimization space than by increasing the number of free parameters; ultimately, this discourages over-parameterization by showing that part of the apparent need for many free parameters can be replaced by structured reuse.

\textbf{Limitations and Future Directions.}\label{sec:limitations}
While Kolmogorov complexity proxies such as BDM and RAS serve as useful surrogates in this work, they can occasionally be misleading since they are primarily defined for binary objects. Future research should focus on designing proxies specifically adapted for deep learning, such as BDM variants that employ high-precision quantization and more sophisticated approximations of algorithmic probability instead of CTM. 

The algorithmic simplification achieved by MoMos does not currently translate into improved hardware utilization. This is not an inherent limitation of the method, but rather a reflection of current hardware acceleration libraries which lack support for fast, reusable block operations. Adapting techniques like fast lookup table 
quantization schemes~\citep{li2025lut} could significantly enhance the hardware efficiency of MoMos.

\vspace{-0.3cm}
\section{Conclusion}
\vspace{-0.2cm}
Over-parameterized deep learning models lend themselves well to compression, which is often attributed to a byproduct of scale. Model compression techniques exploit the redundancies in the neural network parameterization to significantly reduce model size, while retaining most of the performance. In this work, we provide novel insights into the significant redundancies found in neural networks. To this end, we draw upon principles from algorithmic complexity theory, which posits that truly random objects are incompressible, while non-random models are compressible. Applying this principle, we demonstrate the algorithmic complexity of randomly initialized neural networks is higher than that of the same model at convergence.

This reduction in complexity suggests an increased occurrence of patterns, or more simply, structure, in the parameterization of neural networks; supporting the claim that learning induces structure and that such structure is compressible. Inspired by BDM, we propose the Mosaic-of-Motifs (MoMos) technique, which constraints the optimization space, such that model parameterizations are more compressible with provably lower worst-case algorithmic complexity. We show that MoMos parameterized networks retain comparable performance, support fine-tuning of larger pretrained models, and admit model-agnostic compressibility. These insights offer a valuable direction for explaining redundancies in deep learning and may foster new directions for future investigation.

\paragraph{Broader Impact Statement.} This paper presents work whose goal is to advance the field of machine learning. There are many potential societal consequences of our work, none of which we feel must be specifically highlighted here.

\paragraph{Generative AI Usage Statement}

ChatGPT version 5.4 and Google Gemini were used to support programming tasks, including the development of scripts for visualization and setting up experiments, to edit and refine language and grammar in selected sections of the manuscript.

\begin{ack}
All authors acknowledge funding from the European Union’s Horizon Europe Research and Innovation Action program under grant agreements No. 101070284, No. 101070408, and No. 101189771. SW and RS acknowledge funding from the Independent Research Fund Denmark (DFF) under grant agreement No. 4307-00143B. RS also acknowledges funding received under Independent Research Fund Denmark (DFF) under grant agreement number 4307-00143B. The authors thank members of \href{https://saintslab.github.io/}{SAINTS Lab} for valuable discussions.
\end{ack}

\bibliographystyle{\PaperBibliographyStyle}
\bibliography{references}

\clearpage
\appendix
\section{MoMos: Theoretical Details and Proofs}\label{app:proofs}
\subsection{Proof of Proposition \ref{prop:momo_kc} (MoMos Complexity Bound)}\label{app:proof_upper_momo}
Let $f_\W$ be a model parameterized by weights $\W \in \mathbf{R}^n$. Choose a deterministic procedure for an encoder $\enc_q$, fixing weights to $q$-bit representations. The encoder is fixed in the main text, but we formally account for the (constant) cost of specifying the encoding. The bit-string representing the weights is given by $\W \in \bits^{nq}$. Fix a block size $s \in \N$, a motif count $k \in \N$, and precision $q \in \N$.  Then define a slicing operator $\psi$ with inverse $\psi^{-1}$ satisfying $\psi^{-1}(\psi(\W)) = \W$, with $m = \ceil{n/s}$ denoting the number of blocks obtained by slicing $\W$. 

Let $\What \in \mathcal{H}_{S,K}$. By Definition \eqref{eq:def_HSK}, there exist motifs $\mathcal{Z} = (\zz_1,\dots,\zz_K)$ with $\zz_k \in \R^s$ and a mosaic $\mathcal{M} \in \{1,\dots,k\}^m$ such that
\begin{equation}
\widehat{\W} = \psi^{-1}(\zz_{\mathcal{M}(1)},\dots,\zz_{\mathcal{M}(m)}).
\label{eq:W_hat_representation_app}
\end{equation}
Recall that we use $\mathcal{K}(\What)$ to denote the Kolmogorov complexity of the fixed-precision encoding, i.e., $\mathcal{K}(\What) = \mathcal{K}(\enc_q(\What))$. We will upper bound $\mathcal{K}(\widehat{\W}) =\mathrm{\enc}_q(\widehat{\W})$ by explicitly describing a program that prints $\mathrm{enc}_q(\widehat{\W})$ and halts.

Consider a program $\mathbf{p}$ that consists of a fixed decoding and reconstruction routine, followed by a finite description that encodes the motifs and the mosaic. Since $s, k, q$, and $n$ are fixed for the $\mathcal{H}_{s,k}$, the routine can compute $m=\ceil{n/s}$ and parse input of predetermined length in constant time. For each $\zz_k \in \R^s$, the program encodes the motifs with  $\mathrm{enc}_q$ to a binary string of length $sq$. Concatenating the $k$ motif encodings yields a motif description of length $ksq$ bits. Then each entry in the motif $\mathcal{M}(i)$ takes values in $\{1,\dots,k\}$. The program encodes each value in this set by a fixed-length binary representation of length $\ceil{\log_2 k}$ bits. For example, writing $\mathcal{M}(i)-1$ in binary with padding. Concatenating the $m$ encoding yields a pattern description of length $m\ceil{\log_2 k}$ bits.

The program parses the first $ksq$ description bits as $k$ consecutive blocks of length $sq$, interpretating them as motifs $\mathrm{enc}_q(\zz_1),\dots,\mathrm{enc}_q(\zz_K)$, and then parses the next $m\ceil{\log_2 k}$ bits as $m$ encodings of length $\ceil{\log_2 k}$, producing the mosaic indices $\mathcal{M}(1),\dots,\mathcal{M}(m)$. It then forms the sequence of encoded blocks
\begin{equation}
\mathrm{enc}_q(\zz_{\mathcal{M}(1)}),\ \mathrm{enc}_q(\zz_{\mathcal{M}(2)}),\ldots,\mathrm{enc}_q(\zz_{\mathcal{M}(m)}) .
\end{equation}
Moreover, concatenates them to obtain the padded encoding length $ms$. Finally, $\mathbf{p}$ discards the last $(ms-n)q$ bits to satisfy $\psi^{-1}$. The result is exactly $\mathrm{enc}_q(\widehat{\W})$, which the routine prints before halting.

Let $\zeta$ denote the length in bits of this routine, including any constant overhead required by the reference universal machine to interpret the description format. The constant $\zeta$ does not depend on $\What$, hence the the full program length satisfies
\begin{equation}
|\mathbf{p}| \le ksq + m\ceil{\log_2 k} + \zeta.
\end{equation}

Since $\mathbf{p}$ outputs $\mathrm{enc}_q(\widehat{\W})$ and halts, the definition of Kolmogorov complexity gives
\begin{equation}
\mathcal{K}(\widehat{\W}) = \mathcal{K}(\mathrm{enc}_q(\What)) \le |\textbf{p}|
\le ksq + m\ceil{\log_2 k} + \zeta,
\end{equation}
which concludes the proof of \ref{prop:momo_kc}. \hfill $\square$ 

\subsection{Proof of Lemma \ref{lemma:linear_component} (Optimization Subspaces)}\label{app:linear_component}
Fix a mosaic $\mathcal{M} \in [k]^m$. Recall that $\mathcal{H}_{s,k} \subseteq \mathbb{R}^n$ consist of all $\What$ whose block decomposition $\psi(\What)=(\widehat{\mathbf{b}}_1,\ldots,\widehat{\mathbf{b}}_m)$ satisfies $\widehat{\mathbf{b}}_i = \widehat{\mathbf{b}}_j$ whenever $\mathcal{M}(i) = \mathcal{M}(j)$, equivalently $\widehat{\mathbf{b}}_i = \mathbf{z}_{\mathcal{M}(i)}$ for some choice of motifs $\mathcal{Z} = (\mathbf{z}_1,\ldots, \mathbf{z}_k)$.   

Let $z = (\mathbf{z}_1,\ldots, \mathbf{z}_k) \in \mathbb{R}^{ks}$ denote the concatenation of all motifs. Define $F_{\mathcal{M}}: \mathbb{R}^{ks} \to \mathbb{R}^{n}$ by
\begin{equation}
    F_\mathcal{M}(z) = \psi^{-1}(\zz_{\mathcal{M}(1)},\dots,\zz_{\mathcal{M}(m)}).
\end{equation}
By construction, as $z$ ranges over $\R^{ks}$, the image $\mathrm{Im}(F_\mathcal{M})$ is exactly the set of vectors obtained by holding $\mathcal{M}$ fixed and varying motifs, which is $\mathcal{H}_{s,k}(\mathcal{M})$.

The operation $z \to (\mathbf{z}_i,\ldots, \mathbf{z}_k)$ is linear in $z$ because it only selects and repeats blocks, and $\psi^{-1}$ is the composition of concatenation with fixed truncations that discard the padded bits; both are linear maps. Hence $\mathcal{H}_{s,k}(\Mcal) = \mathrm{Im}({F_{\mathcal{M}}})$ is a linear subspace of $\mathbb{R}^n$ and
\begin{equation}
    \dim(\mathcal{H}_{s,k}(\mathcal{M})) = \mathrm{rank}(F_\mathcal{M}) \le ks .
\end{equation}
When $s \mid n$, let $J_\mathcal{M} = \{\mathcal{M}(i),\ldots,\mathcal{M}(m) \}$ and $k_\mathcal{M} = |J_\mathcal{M}|$. Then only motifs with indices in $J_\mathcal{M}$ affect $F_\mathcal{M}(z)$. Motifs outside $J_\mathcal{M}$ are unused, thus $\mathrm{rank}(F_\mathcal{M}) \le sk_\mathcal{M}$ and $\psi^{-1}$ merely concatenates. For each index $j \in J_\mathcal{M}$ pick an index $i(j) \in \{1,\ldots,m \}$ such that $\mathcal{M}(i(j))=j$. Then motif $\widehat{\mathbf{b}}_{i(j)}$ of $F_\mathcal{M}(z)$ equals $\mathbf{z}_j$. Therefore, if $F_{\mathcal{M}}(z) = 0$, every used motif $\mathbf{z}_j$ must equal 0, which shows that the induced map from the $sk_\mathcal{M}$ ranks of the used motifs to $\mathcal{H}_{s,k}$ is one-to-one. Hence, $\mathrm{rank}(\mathcal{F}_\mathcal{M}) = sk_\mathcal{M}$ and
\begin{equation}
    \dim(\mathcal{H}_{s,k}(\mathcal{M})) = sk_\mathcal{M} \le sk \quad \text{when~} s \mid n .
\end{equation}
When $s \nmid n$, truncation in $\psi^{-1}$ can reduce rank, so only the upper bound $\dim(\mathcal{H}_{s,k}(\mathcal{M})) \le sk_\mathcal{M} \le sk$ is guaranteed. \hfill $\square$

\subsection{Proof of Proposition \ref{prop:distinctness} (Distinctness Criterion of Mosaics)}\label{app:distinctness_criterion}
Assume $s \mid n$ so that $m = n/s$ and $\psi^{-1}$ concatenates without truncating. Fix $\mathcal{M} \in \{1,\ldots,k\}^m$. Then by definition of $\mathcal{H}_{s,k}(\mathcal{M})$, the parameterization $\What \in \mathbb{R}^n$ belong to $\mathcal{H}_{s,k}(\mathcal{M})$ if an only if its blocks $\psi(\What) = (\widehat{\mathbf{b}}_1, \ldots, \widehat{\mathbf{b}}_m)$ satisfy $\widehat{\mathbf{b}}_i = \widehat{\mathbf{b}}_j$ whenever $\mathcal{M}(i) = \mathcal{M}(j)$. Thus $\mathcal{H}_{s,k}(\mathcal{M})$ is determined exactly by the partition $\mathcal{P}(\mathcal{M})$ of $\{1,\ldots,m \}$ into the non-empty parts $\{i : \mathcal{M}(i) = j \}$. If $\mathcal{P}(\mathcal{M}) = \mathcal{P}(\mathcal{M}')$, then the equalities enforces for blocks are the same for $\mathcal{M}$ and $\mathcal{M}'$, so the constrains coincide when $\mathcal{H}_{s,k}(\mathcal{M}) = \mathcal{H}_{s,k}(\mathcal{M'})$. Conversely, if $\mathcal{P}(\mathcal{M}) \neq \mathcal{P}(\mathcal{M}')$, then there exists $i,j \in \{1,\ldots,m \}$ such that $\mathcal{M}(i) = \mathcal{M}(j)$ but $\mathcal{M'}(i) \neq \mathcal{M'}(j)$. Choose blocks such that $\widehat{\mathbf{b}}_i = \widehat{\mathbf{b}}$ and all other blocks pairwise distinct. Concatenating these blocks defines a parameterization $\What$ that satisfies the block equalities for one mosaic, but violates them for the other. Hence $\What$ belongs to exactly one of $\mathcal{H}_{s,k}(\mathcal{M})$ and $\mathcal{H}_{s,k}(\mathcal{M'})$, so the two sets differ. \hfill $\square$.

\section{Additional Experimental Details}\label{app:exp_setup}
\subsection{Datasets and Tasks}
\paragraph{Full Training Setting.}
We use the MoMos parameterization in Algo.~\ref{algo:momos} across three different neural network architectures: ResNet20~\citep{resnet}, Tiny-ViT~\citep{tinyvit}, and a multi-layer perceptron (MLP) with 5 hidden layers and 256 neurons per layer. All models are trained and evaluated on CIFAR-10~\citep{cifar10}, where $5\%$ of the training set is held out for validation, giving a $47.5\mathrm{K}$/$2.5\mathrm{K}$/$10\mathrm{K}$ train/validation/test split with a batch size of 128. We augment the dataset during training with random horizontal flip and random crop with 4-pixel padding, and use RandAugment~\citep{torchvision2016}, which applies two randomly sampled augmentations per image. All models are trained from scratch for 200 epochs.

For the MLP and Tiny-ViT, we use AdamW~\citep{adamw} with learning rate $3\times10^{-4}$ and weight decay $10^{-2}$, with Tiny-ViT additionally using images resized to $224\times224$. For ResNet20, we train with SGD using learning rate $0.1$, momentum $0.9$, and weight decay $10^{-4}$. 

\paragraph{Fine-tuning Setting.}
In the fine-tuning setting, we use MoMos with a pretrained Swin Base model~\citep{liu2021swin} on Tiny ImageNet~\citep{le2015tiny}, and with SmolLM2-360M~\citep{allal2025smollm2} on WikiText-103~\citep{merity2016pointer}. For Swin Base, we use a $95\mathrm{K}$/$5\mathrm{K}$/$10\mathrm{K}$ train/validation/test split, where $5\%$ of the original training set is held out for validation, and the official validation split is used as the test set. Images are processed at $224\times224$, and the model is fine-tuned for 30 epochs using AdamW with learning rate $10^{-5}$, weight decay $10^{-8}$, batch size 128, and no learning-rate decay.

For SmolLM2-360M, we fine-tune under the causal next-token prediction objective using the official WikiText-103 train, validation, and test splits. We tokenize the raw text with the model tokenizer, concatenate the tokens, and group them into sequences of length $8192$, giving $14432$ training sequences, $31$ validation sequences, and $35$ test sequences in our setup. The model is fine-tuned for 3 epochs using AdamW with learning rate $2\times10^{-5}$, weight decay $10^{-2}$, batch size 2, gradient accumulation over 8 steps, and gradient checkpointing with half-precision floats. In the language-model setting, the QAT baseline uses the straight-through estimator described in Appx.~\ref{app:baseline_details}, excluding embeddings, normalization layers, and the language model head from fake quantization.

\begin{algorithm}[tbp]
\caption{\textsc{MoMos} Optimization}
\label{algo:momos}
\begin{algorithmic}[1]
\INPUT initialization $\W_0\in\mathbb{R}^n$, block size $s$, model capacity $c\in[0,1]$, iterations $n_T$, data $\mathcal{D}$, loss $\ell$, optimizer \textsc{Opt}
\OUTPUT $\widehat{\W}^{(n_T)}\in\mathcal{H}_{s,k}$, motifs $\mathcal{Z}^{(n_T)}$, mosaic $\Mcal^{(n_T)}$
\STATE $m \gets \ceil{n/s}$,\quad $k \gets \max\{1,\floor{c\cdot m}\}$
\FOR{$t=1,\ldots,n_T$}
    \STATE $\W^{(t)} \gets \textsc{OptStep}(\W_{t-1};\mathcal{D},\ell)$
    \STATE $(\bb_1^{(t)},\ldots,\bb_m^{(t)}) \gets \psi(\W^{(t)})$
    \STATE $\zz_1^{(t)} \gets \mathbf{0}_s$
    \IF{$k > 1$}
        \STATE Uniformly sample distinct indices $j_1,\ldots,j_{k-1}$ from $[m]$
        \FOR{$r=2,\ldots,k$}
            \STATE $\zz_r^{(t)} \gets \bb_{j_{r-1}}^{(t)}$
        \ENDFOR
    \ENDIF
    \STATE $\mathcal{Z}^{(t)} \gets (\zz_1^{(t)},\ldots,\zz_k^{(t)})$
    \FOR{$i=1,\ldots,m$}
        \STATE $\Mcal^{(t)}(i) \gets \arg\min_{r\in\{1,\ldots,k\}} \|\bb_i^{(t)} - \zz_r^{(t)}\|_2^2$
    \ENDFOR
    \STATE $\widehat{\W}^{(t)} \gets \psi^{-1}\bigl(\zz_{\Mcal^{(t)}(1)}^{(t)},\ldots,\zz_{\Mcal^{(t)}(m)}^{(t)}\bigr)$
    \STATE $\W_t \gets \widehat{\W}^{(t)}$
\ENDFOR
\end{algorithmic}
\end{algorithm}

\subsection{Comparison to Existing Compression Baselines}\label{app:baseline_details}
Besides full-precision training, we compare against five additional baselines. These are scalar k-means quantization (VQ) and product quantization (PQ)~\citep{gong2014compressing}, look-up table quantization (LUT-Q)~\citep{Cardinaux_2020}, Soft Weight Sharing (Soft-WS)~\citep{ullrich2017soft}, and weight-only quantization-aware training (QAT)~\citep{qat}. For VQ we use scalar k-means with $k=16$, 10 Lloyd iterations, and 1 restart. PQ uses $x$-axis product quantization with $k=8$ codewords per segment, segment dimension 2, and 10 Lloyd iterations. LUT-Q uses a dictionary size $K=16$ with $M=1$ assignment-update. Soft-WS uses 16 mixture components with $\pi_0=0.99$ and $\tau=0.003$. QAT is evaluated at $q\in{16,8,4,2}$ bits per weight using symmetric per-tensor fake quantization and the straight-through estimator for the rounding operation~\citep{bengio2013ste}.

For the non-MoMos baselines, we report the ratio between the original FP32 model size and the theoretical deployed size in bits, computed from the compression calculations described in the respective works. This is important because, unlike MoMos, these baselines are not equally model-agnostic. For example, VQ and PQ only target linear layers in our implementation, LUT-Q targets convolutional and linear layers, and Soft-WS targets all trainable parameters. So, depending on the architecture, the same baseline may compress most of the model or only a very small part of it.

Furthermore, for QAT at bitwidth $q$, we use the deployed bit-length of $nq$, corresponding to a storage reduction of $32/q$ relative to the original FP32 model. For VQ, parameters that are not quantized remain at 32 bits, and each quantized tensor stores a scalar codebook plus one centroid index per weight, giving $32n + \sum_i \left(32k_i + n_i\lceil \log_2 k_i \rceil\right)$ bits. LUT-Q uses the same storage formula in our implementation, but over its own set of layers. For PQ, parameters that are not quantized again remain at 32 bits, while each quantized tensor stores one smaller codebook per segment position and one codeword index per segment, giving $32n + \sum_i \left(32m_i k d + v_i m_i\lceil \log_2 k \rceil\right)$ bits. For Soft-WS with $K$ shared scalar values, we use $32n + 32K + \hat{n}\lceil \log_2 K \rceil$ bits where $n,\hat{n}$ are the full FP32 and shared parameters, respectively.

\subsection{Profiling Training Time and Memory Overhead}
Timing and memory overheads are measured separately from the main experiments. We perform a 20-epoch profiling run on a single NVIDIA A100 GPU and treat the first 5 epochs as warm-up. Reported timing and memory overheads are computed relative to the corresponding FP32 baseline within the same profiling suite. Accuracy values are not taken from these profiling runs; they are taken from the full convergence runs reported in Appx.~\ref{tab:results_main}.

\subsection{Compute Resources.}\label{sec:compute_resources}
All experiments described above are each run thrice over different seeds.Across the completed settings for MLP, ResNet20, Tiny-ViT, Swin-B, and SmolLM2, the experiments corresponds to approximately 587 single-GPU wall-clock hours. Including failed runs and preliminary experiments, we estimate an additional 25 single-GPU hours. Runs were performed on a shared cluster using NVIDIA L40S, A100 40GB/80GB, A40, H100 PCIe, and Quadro RTX 6000 GPUs. Most of these runs were on the latter type of GPU.

\subsection{Loss Landscape Visualization.}\label{app:exp_loss}
To examine where the MoMos solutions lie relative to the unconstrained solution, we visualize a local two-dimensional slice of the loss landscape, following~\cite{li2018visualizing}. Concretely, for each architecture we take the converged baseline parameters $\theta_0$ and two converged MoMos parameters, $\theta_{s=2}$ and $\theta_{s=4}$, at fixed capacity $c=0.1$. We then evaluate the validation loss on the grid
\begin{equation}
    \theta(\alpha,\beta) = \theta_0 + \alpha(\theta_{s=2}-\theta_0) + \beta(\theta_{s=4} \theta_0),
\end{equation}
where $\alpha$ and $\beta$ move from the baseline solution toward the two MoMos solutions. That is, $(\alpha,\beta)=(0,0)$ gives the unconstrained baseline, $(1,0)$ gives the MoMos solution with $s=2$, and $(0,1)$ gives the MoMos solution with $s=4$. We then sample grid points in this plane, replace the model weights by $\theta(\alpha,\beta)$, and evaluate the loss of the resulting parameterization on the validation set.

\section{Additional Results}\label{app:additional_results}
\begin{table*}[htb]
\centering
\tiny
\setlength{\tabcolsep}{1.5pt}
\renewcommand{\arraystretch}{1.0}
\begin{tabular}{lllc lll lll lll}
\toprule
\textbf{Method} & \textbf{Cap. (\%)} & \textbf{Config.} &
& \multicolumn{3}{c}{\textbf{MLP}}
& \multicolumn{3}{c}{\textbf{ResNet20}}
& \multicolumn{3}{c}{\textbf{Tiny-ViT}} \\
\cmidrule(lr){5-7}\cmidrule(lr){8-10}\cmidrule(lr){11-13}
& & &
& \textbf{Acc. (\%) $\uparrow$} & \textbf{RAS} $\uparrow$ & $\mathbf{r}_{\mathbf{BDM}}$ $\downarrow$
& \textbf{Acc. (\%) $\uparrow$} & \textbf{RAS} $\uparrow$ & $\mathbf{r}_{\mathbf{BDM}}$ $\downarrow$
& \textbf{Acc. (\%) $\uparrow$} & \textbf{RAS} $\uparrow$ & $\mathbf{r}_{\mathbf{BDM}}$ $\downarrow$ \\
\midrule
Baseline & 100 & $q=32$ &&
$56.40 \pm 1.70$ & 1.00 & 0.96 &
$92.20 \pm 0.50$ & 1.00 & 0.90 &
$86.60 \pm 0.20$ & 1.00 & 1.00 \\
\midrule
QAT & -- & $q=16$ &&
$55.80 \pm 3.80$ & 2.00 & 0.97 &
$92.40 \pm 0.30$ & 2.00 & 0.90 &
$87.00 \pm 0.20$ & 2.00 & 1.00 \\
QAT & -- & $q=8$ &&
$57.30 \pm 1.90$ & 4.00 & 0.96 &
$92.30 \pm 0.50$ & 4.00 & 0.90 &
$87.10 \pm 0.20$ & 4.00 & 1.00 \\
QAT & -- & $q=4$ &&
$55.00 \pm 3.20$ & 8.00 & 0.86 &
$91.80 \pm 0.30$ & 8.00 & 0.74 &
$86.90 \pm 0.30$ & 8.00 & 0.97 \\
QAT & -- & $q=2$ &&
$31.50 \pm 1.60$ & 16.00 & 0.07 &
$13.40 \pm 2.50$ & 16.00 & 0.02 &
$74.20 \pm 0.30$ & 16.00 & 0.54 \\
\midrule
MoMos & 0.05 & $s=2$ &&
$54.10 \pm 3.10$ & 7.97 & 0.66 &
$66.20 \pm 2.80$ & 9.10 & 0.07 &
$72.10 \pm 0.80$ & 5.80 & 1.00 \\
MoMos & 0.10 & $s=2$ &&
$56.60 \pm 0.20$ & 7.06 & 0.85 &
$81.50 \pm 0.50$ & 7.94 & 0.57 &
$76.90 \pm 2.50$ & 5.30 & 1.00 \\
MoMos & 5 & $s=2$ &&
$56.60 \pm 1.10$ & 3.52 & 0.96 &
$91.60 \pm 0.90$ & 3.95 & 0.90 &
$87.00 \pm 0.20$ & 3.02 & 1.00 \\
MoMos & 10 & $s=2$ &&
$53.40 \pm 2.40$ & 2.86 & 0.97 &
$91.70 \pm 0.70$ & 3.14 & 0.90 &
$86.80 \pm 0.20$ & 2.52 & 1.00 \\
\midrule
MoMos & 0.05 & $s=4$ &&
$48.80 \pm 6.70$ & 18.12 & 0.15 &
$10.00 \pm 0.00$ & 21.11 & 0.01 &
$45.80 \pm 5.40$ & 12.72 & 0.99 \\
MoMos & 0.10 & $s=4$ &&
$48.50 \pm 6.90$ & 15.75 & 0.33 &
$17.80 \pm 4.40$ & 17.96 & 0.01 &
$51.00 \pm 6.60$ & 11.50 & 1.00 \\
MoMos & 5 & $s=4$ &&
$55.70 \pm 0.60$ & 6.28 & 0.91 &
$87.70 \pm 2.90$ & 6.96 & 0.58 &
$85.90 \pm 0.20$ & 5.47 & 1.00 \\
MoMos & 10 & $s=4$ &&
$56.20 \pm 0.50$ & 4.60 & 0.91 &
$90.20 \pm 0.40$ & 4.96 & 0.63 &
$86.70 \pm 0.40$ & 4.16 & 1.00 \\
\midrule
VQ & -- & $k=16$ &&
$47.40 \pm 9.40$ & 7.94 & 0.97 &
$92.40 \pm 0.00$ & 1.00 & 0.90 &
$86.50 \pm 0.00$ & 6.28 & 1.00 \\
PQ & -- & $k=8,d=2$ &&
$42.80 \pm 5.90$ & 12.33 & 0.96 &
$92.80 \pm 0.00$ & 1.00 & 0.90 &
$60.60 \pm 0.00$ & 9.28 & 1.00 \\
LUT-Q & -- & $K=16,M=1$ &&
$58.20 \pm 0.00$ & 7.94 & 0.96 &
$92.30 \pm 0.00$ & 7.65 & 0.85 &
$86.60 \pm 0.00$ & 7.35 & 1.00 \\
Soft-WS & -- & $K=16,\tau=0.003$ &&
$10.00 \pm 0.00$ & 8.00 & 0.09 &
$57.50 \pm 0.00$ & 8.00 & 0.94 &
$61.70 \pm 7.80$ & 8.00 & 0.74 \\
\bottomrule
\end{tabular}
\caption{\everypar{\looseness=-1}CIFAR-10 results for MoMos and several compression baselines. We report test accuracy, relative algorithmic simplification (RAS), and the BDM ratio computed on the binarized weights as in Eq.~\eqref{eq:bdm}. MoMos configurations vary block size $s$ and capacity $c$, while baseline compression methods are evaluated using the settings described in Appx.~\ref{app:exp_setup}.}
\label{tab:results_main}
\end{table*}

\begin{table*}[htb]
\centering
\tiny
\setlength{\tabcolsep}{4pt}
\renewcommand{\arraystretch}{1.00}
\begin{tabular}{lllc rr rr rr}
\toprule
\textbf{Method} & \textbf{Cap. (\%)} & \textbf{Config.} &
& \multicolumn{2}{c}{\textbf{MLP}}
& \multicolumn{2}{c}{\textbf{ResNet20}}
& \multicolumn{2}{c}{\textbf{Tiny-ViT}} \\
\cmidrule(lr){5-6}\cmidrule(lr){7-8}\cmidrule(lr){9-10}
& & &
& \textbf{Time $\Delta$} & \textbf{Mem. $\Delta$}
& \textbf{Time $\Delta$} & \textbf{Mem. $\Delta$}
& \textbf{Time $\Delta$} & \textbf{Mem. $\Delta$} \\
\midrule
Baseline & 100 & $q=32$ &&
$0.00$ & $0.00$ &
$0.00$ & $0.00$ &
$0.00$ & $0.00$ \\
MoMos & 10 & $s=2$ &&
$5.30$ & $203.02$ &
$-3.10$ & $0.00$ &
$3.29$ & $-0.00$ \\
MoMos & 10 & $s=4$ &&
$3.00$ & $105.92$ &
$-2.99$ & $0.00$ &
$1.85$ & $-0.00$ \\
\midrule
VQ & -- & $k=16$ &&
$3.27$ & $294.49$ &
$-2.45$ & $0.00$ &
$-0.30$ & $0.00$ \\
PQ & -- & $k=8,d=2$ &&
$12.94$ & $0.10$ &
$-2.96$ & $0.00$ &
$7.78$ & $0.00$ \\
LUT-Q & -- & $K=16,M=1$ &&
$5.64$ & $330.45$ &
$58.86$ & $1.07$ &
$24.54$ & $2.75$ \\
Soft-WS & -- & $K=16,\tau=0.003$ &&
$7.02$ & $565.39$ &
$185.11$ & $0.02$ &
$66.02$ & $25.14$ \\
\bottomrule
\end{tabular}
\caption{Training-time and memory overhead between MoMos and baselines compression methods. Overheads are reported as percentages relative to the FP32baseline for the same architecture, measured in separate 20-epoch profiling runs with the first 5 epochs treated as warm-up epochs. Methods with projection steps use the Flash projection approach from~\cite{yang2026flash} for consistency.}
\label{tab:results_overhead}
\end{table*}

\begin{table*}[htb]
\centering
\tiny
\setlength{\tabcolsep}{4pt}
\renewcommand{\arraystretch}{1.05}
\begin{tabular}{llc lll lll}
\toprule
\textbf{Method} & \textbf{Cap. (\%)} & \textbf{Config.} &
\multicolumn{3}{c}{\textbf{SmolLM2 (360M params.)}} &
\multicolumn{3}{c}{\textbf{Swin-B (86M params.)}} \\
\cmidrule(lr){4-6}\cmidrule(lr){7-9}
& & &
\textbf{PPL $\downarrow$} & \textbf{RAS $\uparrow$} & &
\textbf{Acc. (\%) $\uparrow$} & \textbf{RAS $\uparrow$} & \\
\midrule
Baseline & 100 & $q=32$ &
$9.55 \pm 0.00$ & 1.00 & &
$90.35 \pm 0.17$ & 1.00 & \\
\midrule
QAT & -- & $q=16$ &
$9.55 \pm 0.00$ & 2.00 & &
$90.40 \pm 0.20$ & 2.00 & \\
QAT & -- & $q=8$ &
$10.89 \pm 0.00$ & 4.00 & &
$90.30 \pm 0.10$ & 4.00 & \\
\midrule
MoMos & 5  & $s=2$ &
$9.86 \pm 0.09$ & 2.35 & &
$90.38 \pm 0.07$ & 2.54 & \\
MoMos & 10 & $s=2$ &
$9.74 \pm 0.06$ & 2.04 & &
$90.61 \pm 0.06$ & 2.18 & \\
\midrule
MoMos & 5  & $s=4$ &
$15.38 \pm 1.23$ & 4.35 & &
$81.69 \pm 0.96$ & 4.67 & \\
MoMos & 10 & $s=4$ &
$13.71 \pm 3.29$ & 3.48 & &
$84.34 \pm 1.20$ & 3.68 & \\
\bottomrule
\end{tabular}
\caption{MoMos fine-tuning results on SmolLM2 and Swin-B. SmolLM2 is fine-tuned for 3 epochs and evaluated with perplexity, while Swin-B is fine-tuned for 30 epochs and evaluated with test accuracy. QAT baselines use weight-only quantization-aware training at the indicated bitwidths.}
\label{tab:ft_results}
\end{table*}

\begin{figure}[htb]
    \centering
    \includegraphics[width=0.9\linewidth]{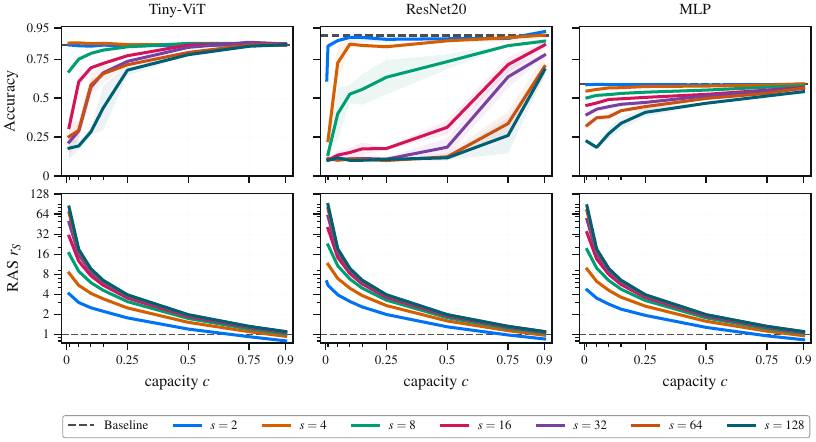}
    \caption{Ablation over capacity and block size for \textbf{(top)} validation accuracy and \textbf{(bottom)} RAS. The dashed horizontal line is the FP32 baseline. Smaller motif blocks achieve baseline accuracy at lower capacity, whereas larger blocks require higher capacity to reach comparable performance.}
    \label{fig:full_ablation}
\end{figure}

\begin{figure}[htb]
    \centering
    \includegraphics[width=1\linewidth]{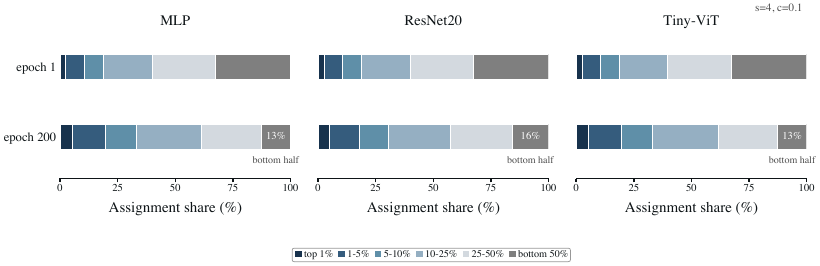}
    \caption{Motif assignment concentration is shown for MoMos with $s$ equal 2 and 4 for $c=0.1$. Motifs are ranked by usage, grouped by percentile, and the bars indicate the fraction of block assignments for each group. The assignment mass shifts toward the most used motifs, while the bottom 50\% receive only a small fraction of assignments at convergence.}
    \label{fig:motif_concentration_s4}
\end{figure}

\begin{figure}[htb]
    \centering
    \includegraphics[width=1\linewidth]{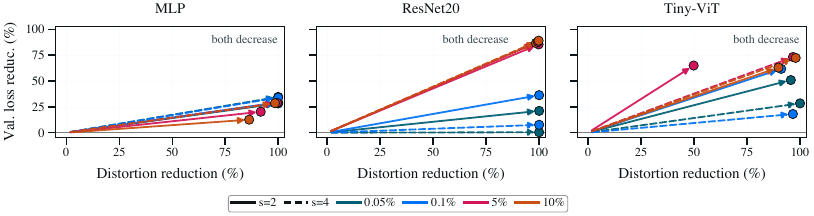}
    \caption{Distortion and validation loss decrease together. Each arrow points from and towards the first epoch and last epoch (at convergence) for a specific MoMos setting. Most settings result in a reduction in both the projection distortion and a lower validation loss.}
    \label{fig:distortion_all}
\end{figure}

\begin{figure}[htb]
    \centering
    \includegraphics[width=1\linewidth]{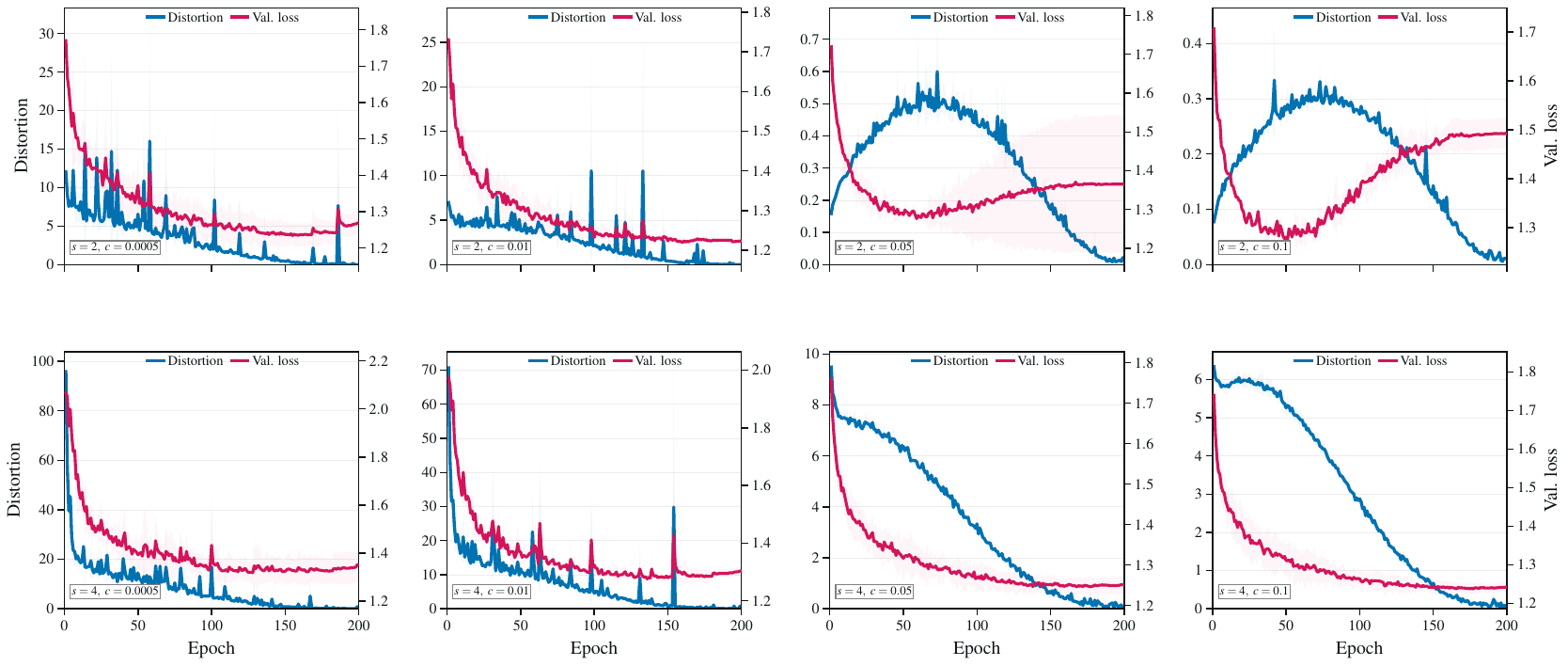}
    \caption{Distortion and validation loss during training for the MLP across MoMos settings. Each panel fixes the block size $s$ and capacity $c$.}
    \label{fig:distortion_mlp}
\end{figure}

\begin{figure}[htb]
    \centering
    \includegraphics[width=1\linewidth]{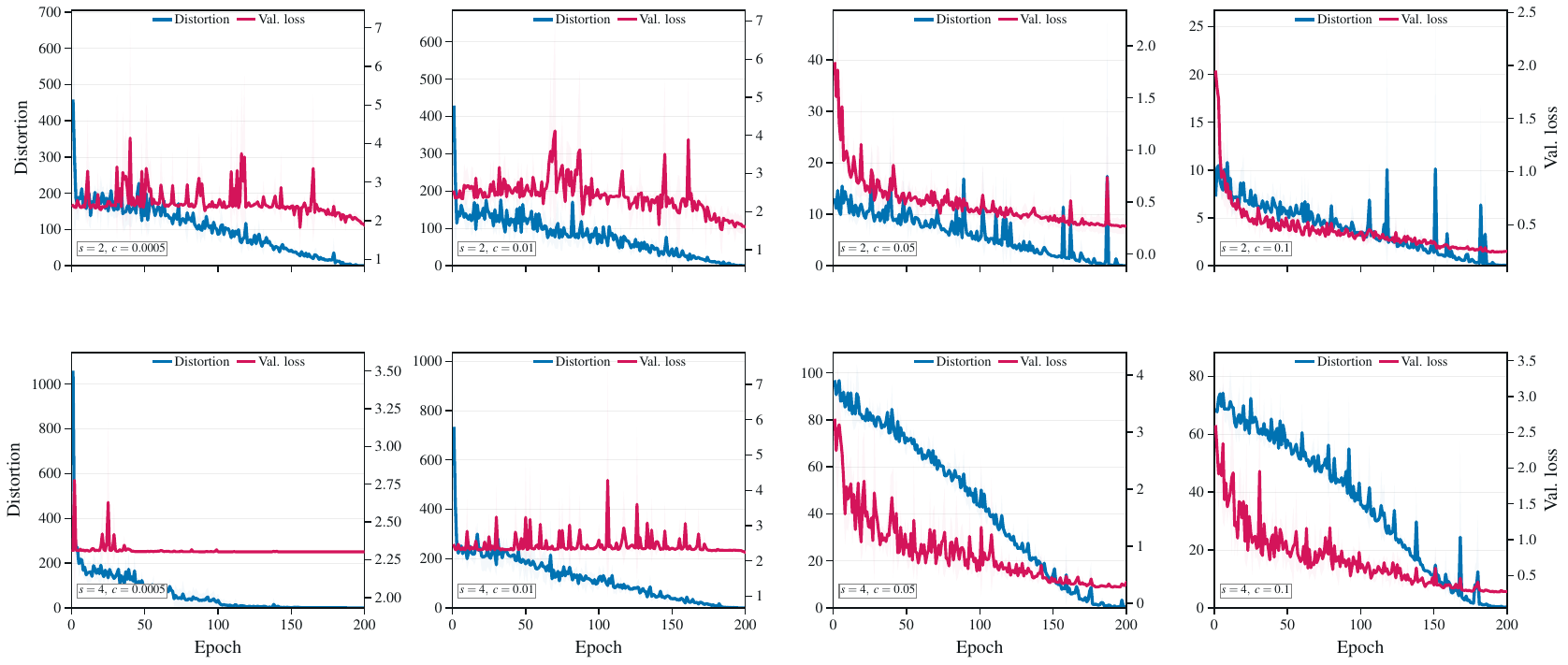}
    \caption{Distortion and validation loss during training for ResNet20 across MoMos settings. Each panel fixes the block size $s$ and capacity $c$.}
    \label{fig:distortion_resnet}
\end{figure}

\begin{figure}[htb]
    \centering
    \includegraphics[width=1\linewidth]{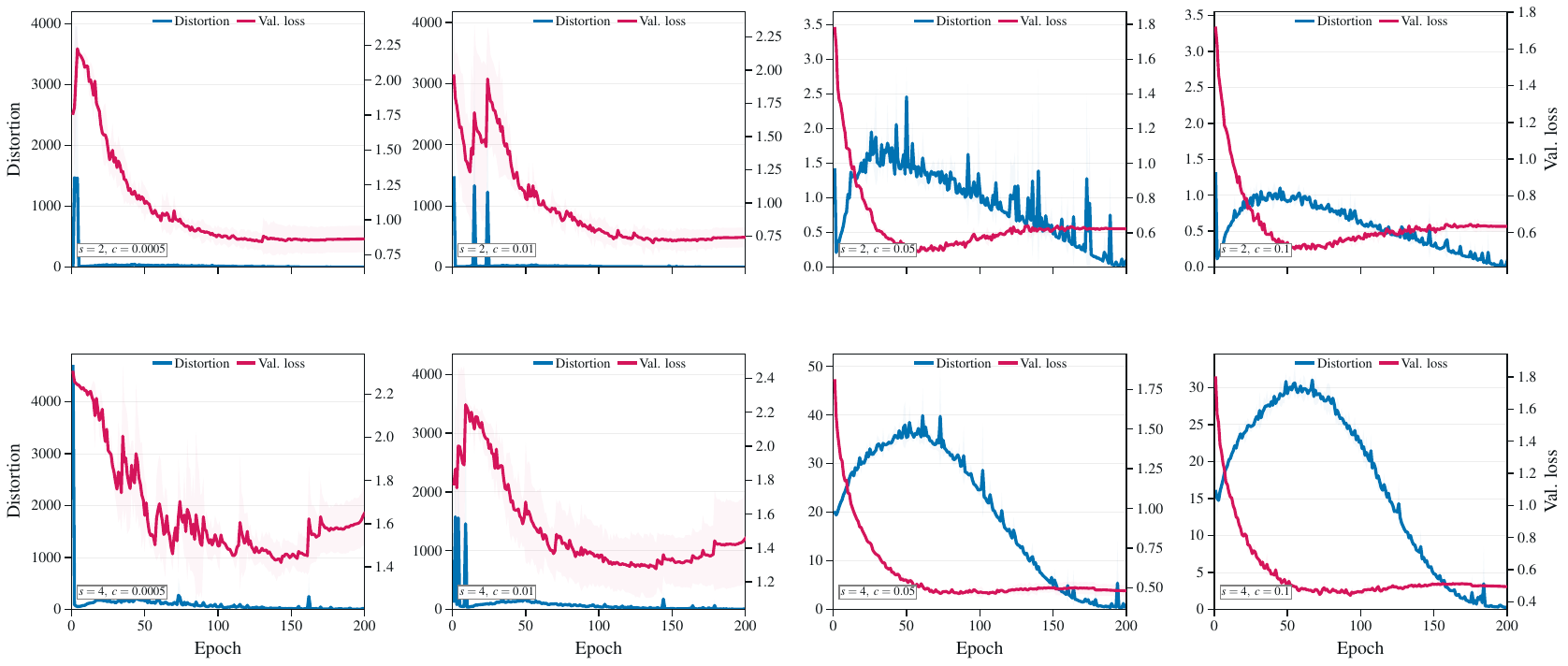}
    \caption{Distortion and validation loss during training for Tiny-ViT across MoMos settings. Each panel fixes the block size $s$ and capacity $c$.}
    \label{fig:distortion_vit}
\end{figure}

\end{document}